\def\year{2019}\relax
\documentclass[letterpaper]{article} 
\usepackage{aaai19}  
\usepackage{times}  
\usepackage{helvet}  
\usepackage{courier}  
\usepackage{url}  
\usepackage{graphicx}  
\usepackage{amsmath,amssymb,amsthm}
\usepackage{trackchanges}
\usepackage{mathtools}
\usepackage{subfigure}
\usepackage{grffile}
\usepackage{thm-restate}

\graphicspath{ {./images/} }

\newcommand{\cL}{\mathcal{L}}
\newcommand{\cS}{\mathcal{S}}

\newcommand{\cA}{\mathcal{A}}

\newcommand{\cT}{\mathcal{T}}

\newcommand{\bR}{\mathbb{R}}

\DeclareMathOperator*{\expect}{{\huge \mathbb{E}}}

\newcommand{\norm}[1]{\left \| #1 \right \|}

\newtheorem{defn}{Definition}

\newtheorem{lem}{Lemma}

\newtheorem{cor}{Corollary}
\newtheorem{thm}{Theorem}

\newcommand{\cbar}{\, | \,}

\def \grad {\nabla}

\DeclareMathOperator*{\argmin}{arg\,min}

\newcommand{\eqnref}[1]{(\ref{eqn:#1})}

\newtheorem{theorem}{Theorem}

\newtheorem{assumption}[theorem]{Assumption}

\addeditor{Carles}
\addeditor{Marc}
\newcommand{\carles}[1]{}
\newcommand{\marc}[1]{}

\def \dpi {{d_{\pi}}}
\def \dmu {{d_{\mu}}}
\def \Ddmu {D_{\dmu}}
\def \Ppi {P_{\pi}}
\def \hPpi {{\hat P}_{\pi}}
\def \hdpi {{\hat d}_{\pi}}
\def \hgamma {\hat \gamma}
\def \Ygamma {Y_{\hgamma}}
\def \Y {Y_{\hgamma}}
\def \Proj {\Pi}  

\def \Vpi {V^\pi}
\def \hVpi {{\hat V}^\pi}
\def \hVpid {{\hat V}^\pi_d}  
\def \cTpi {\cT_\pi}
\def \hV {\hat V}

\def \ratio {\frac{\dpi}{\dmu}}
\def \hratio {\frac{\hdpi}{\dmu}}

\newcommand{\citet}[1]{\citeauthor{#1} \shortcite{#1}}
\newcommand{\citep}[1]{(\citeauthor{#1} \citeyear{#1})}
\newcommand{\citenp}[1]{\citeauthor{#1} \citeyear{#1}}

\begin{document}
\title{Off-Policy Deep Reinforcement Learning by Bootstrapping the Covariate Shift}
\author{Carles Gelada, Marc G. Bellemare\\
Google Brain\\
cgel@google.com, bellemare@google.com}
\maketitle
\begin{abstract}
In this paper we revisit the method of off-policy corrections for reinforcement learning (COP-TD) pioneered by Hallak et al. (2017). Under this method, online updates to the value function are reweighted to avoid divergence issues typical of off-policy learning. While Hallak et al.'s solution is appealing, it cannot easily be transferred to nonlinear function approximation. First, it requires a projection step onto the probability simplex; second, even though the operator describing the expected behavior of the off-policy learning algorithm is convergent, it is not known to be a contraction mapping, and hence, may be more unstable in practice. We address these two issues by introducing a discount factor into COP-TD. We analyze the behavior of discounted COP-TD and find it better behaved from a theoretical perspective. We also propose an alternative soft normalization penalty that can be minimized online and obviates the need for an explicit projection step. We complement our analysis with an empirical evaluation of the two techniques in an off-policy setting on the game Pong from the Atari domain where we find discounted COP-TD to be better behaved in practice than the soft normalization penalty. Finally, we perform a more extensive evaluation of discounted COP-TD in 5 games of the Atari domain, where we find performance gains for our approach. 
\end{abstract}

\section{Introduction}

Central to reinforcement learning is the idea that an agent should learn from experience. While many algorithms learn in a purely online fashion, sample-efficient methods typically make use of past data, viewed either as a fixed dataset, or stored in a replay memory (\citenp{lin93scaling}, \citenp{mnih15human}). Because this past data may not be generated according to the policy currently under evaluation, the agent is said to be learning \emph{off-policy} \citep{sutton18reinforcement}.

By now it is well-documented that off-policy learning may carry a significant cost when combined to function approximation. Early results have shown that estimating the value function off-policy, using Bellman updates, may diverge \citep{baird95residual}, \citep{tsitsiklis97analysis}. More recently, value divergence was perhaps the most significant issue dealt with in the design of the DQN agent \citep{mnih15human}, and remains a source of concern in deep reinforcement learning \citep{vanhasselt16deep}.

Further, under off-policy learning, the quality of the Bellman fixed point suffers as studied by \citet{Kolter2011TheFP} and  \citet{munos03error}. The value function error can be unboundedly large even if the value function can be perfectly approximated. Hence, even in the case where convergence to the fixed point with off-policy data occurs, solutions can be of poor quality. Thus, the existing TD learning algorithms with convergence guarantees under off-policy data \citep{Maei2009ConvergentTL}, \citep{Sutton2009FastGM} can still suffer from off-policy issues.

This paper studies the covariate shift method for dealing with the off-policy problem. The covariate shift method, studied by \citet{hallak17consistent} and \citet{sutton16emphatic},
reweights online updates according to the ratio of the target and behavior stationary distributions. 
Under optimal conditions, the covariate shift method recovers convergent behavior with linear approximation, breaking what \citet{sutton18reinforcement} call the ``deadly triad'' of reinforcement learning. We argue the method is particularly appealing in the context of replay memories, where the reweighting can be replaced by a reprioritization scheme similar to that of \citet{schaul16prioritized}. 

We improve on \citeauthor{hallak17consistent}'s COP-TD algorithm, which has provable guarantees but is difficult to implement in a deep reinforcement learning setting. First, we introduce a discount factor into their update rule to obtain a more stable algorithm. Second, we develop an alternative normalization scheme that can be combined with deep networks, avoiding the projection step necessary in the original algorithm. 
We perform an empirical study of the two methods and their variants on the game Pong from the Arcade Learning Environment and find that our improvements give rise to significant benefits for off-policy learning.

\section{Background}
We study the standard RL setting in which an agent interacts with an environment by observing a state, selecting an action, receiving a reward, and observing the next state. We model this process with a Markov Decision Process $\langle \cS,\cA,R,P,\gamma \rangle$. Here, $\cS$, $\cA$ denote the state and action spaces and $P$ is the transition function. Throughout we will assume that $\cS$ and $\cA$ are finite and write $n := |\cS|$. A policy $\pi$ maps a state to a distribution over actions, $R : \cS \times \cA \to \bR$ is the reward function, and $\gamma$ is the discount factor.

We are interested in the \emph{policy evaluation} problem, where we seek to learn the value function $V^\pi$ of a policy from samples. The value function is the expected sum of discounted rewards from a state when following policy $\pi$:
\begin{equation*}
V^\pi(s) := \expect  \big[ \sum_{t=0}^\infty \gamma^t R(s_t, a_t) | s_0 = s \big],
\end{equation*}
where each action is drawn from the policy $\pi$, i.e. $a_t \sim \pi(\cdot \cbar s_t)$, and states are drawn from the transition function: $s_{t+1} \sim P(\cdot \cbar s_t, a_t)$. We combine the policy $\pi$ and transition function $P$ into a state-to-state transition function $\Ppi$, whose entries are
\begin{equation*}
\Ppi(s' \cbar s) := \sum_{a \in \cA} \pi(a \cbar s) P(s' \cbar s, a) .
\end{equation*}

Let $r_\pi(s) := \expect_{a \sim \pi} R(s,a)$ be the expected reward under $\pi$.
One of the key properties of the value function $V^\pi$ is that it satisfies the \emph{Bellman equation}:
\begin{equation*}
V^\pi(s) = r_\pi(s) + \gamma \expect_{s' \sim \Ppi} V^\pi(s') .
\end{equation*}
In vector notation \citep{puterman94markov}, this becomes
\begin{equation*}
V^\pi = r_\pi + \gamma \Ppi V^\pi,
\end{equation*}
where $V^\pi \in \bR^n$, $r^\pi \in \bR^n$, and $\Ppi \in \bR^{n \times n}$.
The value function is in fact the fixed point of the \emph{Bellman operator} $\cTpi : \bR^n \to \bR^n$, defined as
\begin{equation*}
\cTpi V := r_\pi + \gamma \Ppi V .
\end{equation*}
The Bellman operator describes a single step of dynamic programming \citep{bellman57dynamic} or \emph{bootstrap}; the process $V^{k+1} := \cTpi V^k$ converges to $V^\pi$. More interestingly for us, the operator also describes the expected behavior of learning rules such as temporal-difference learning \citep{sutton88learning} and consequently their learning dynamics \citep{tsitsiklis97analysis}. In the sequel, whenever we analyze the behavior of operators it is with this relationship in mind.

In this paper we consider the process of learning $V^\pi$ from samples drawn from $P$ and a \emph{behavior policy} $\mu$. Following standard usage, we call this process \emph{off-policy learning} \citep{sutton18reinforcement}.

Let $d \in \bR^n$. We write $D_d$ for the corresponding diagonal matrix. For a matrix $A \in \bR^{n \times n}$, the $A$-weighted squared seminorm of a vector $x \in \bR^n$ is $\norm{x}^2_A := \norm{A x}^2 = x^\top A^\top A x$. We specialize this notation to vectors in $\bR^n$ as $\norm{x}^2_d := \sum_{i=1}^n d(i) x(i)^2$.
We write $e$ for the vector of all ones and
$\Delta(\cS)$ for the simplex over states ($d \in \Delta(\cS) \implies d^\top e = 1, d \ge 0)$.
Finally, recall that $d \in \Delta(\cS)$ is the stationary distribution of a transition function $P$ if and only if
\begin{equation*}
d = d P .
\end{equation*}
This distribution is unique when $P$ defines a Markov chain with a single recurrent class (called a unichain, \citenp{meyn12markov}). Throughout we will write $\dpi$ and $\dmu$ for the stationary distributions of $\Ppi$ and $P_\mu$, respectively.

\subsection{Off-Policy Learning with Linear Approximation}

In most practical applications the size of the state space precludes so-called tabular representations, which learn the value of each state separately. Instead, one must approximate the value function. One common scheme is \emph{linear function approximation}, which uses a mapping from states to features
$\phi : \cS \to \bR^k$. The approximate value function at $s$ is then the inner product of a feature vector with a vector of weights $\theta \in \bR^k$:
\begin{equation}\label{eqn:linear_approximation_state}
\hat V(s) = \phi(s)^\top \theta .
\end{equation}
If $\Phi \in \bR^{n \times k}$ denotes the matrix of row-feature vectors, \eqnref{linear_approximation_state} becomes, in vector notation:
\begin{equation*}
\hat V = \Phi \theta .
\end{equation*}

The \emph{semi-gradient update rule} for TD learning \citep{sutton18reinforcement} learns an approximation of $\Vpi$ from sample transitions. Given a starting state $s \in \cS$, a successor state $s' \sim \Ppi(\cdot \cbar s)$, and a step-size parameter $\alpha > 0$, this update is
\begin{equation}\label{eqn:semigradient}
\theta \gets \theta + \alpha \left [ r_\pi(s) + \gamma \phi(s')^\top \theta - \phi(s)^\top \theta \right ] \phi(s) .
\end{equation}
While \eqnref{semigradient} does not correspond to a proper gradient descent procedure (see e.g. \citenp{barnard93temporaldifference}), it can be shown to converge, as we shall now see.

The expected behavior of the semi-gradient update rule is described by the \emph{projected Bellman operator}, denoted $\Proj_d \cTpi$ for some distribution $d \in \Delta(\cS)$ \citep{tsitsiklis97analysis}. The projected Bellman operator is the combination of the usual Bellman operator with a projection $\Proj_d$ in norm $\norm{\cdot}_d$ onto the span of $\Phi$. Typically, the learning rule \eqnref{semigradient} is studied in an online setting, where samples correspond to an agent sequentially experiencing the environment. In the simplest case where an agent follows a single behavior policy $\mu$, this corresponds to $d = \dmu$.

The stationary point of \eqnref{semigradient}, if it exists, is the solution of the \emph{projected Bellman equation}
\begin{equation*}
\hVpi = \Proj_d \cTpi \hVpi .
\end{equation*}
When the projection is performed under the stationary distribution $\dpi$, \eqnref{semigradient} converges to this fixed point provided $\alpha$ is taken to satisfy the Robbins-Monro conditions and other mild assumptions(see \citet{tsitsiklis97analysis}). Taking $d \ne \dpi$, however, may lead to divergence of the weight vector in \eqnref{semigradient}. A sign of the importance of this issue can be seen in \citeauthor{sutton18reinforcement}'s choice to dub ``deadly triad'' the combination of
off-policy learning, function approximation, and bootstrapping.

A prerequisite to guarantee the convergence of \eqnref{semigradient} to $\hVpi$ is that
\begin{equation}
\hV^{k+1} := \Proj_d \cTpi \hV^k
\end{equation}
should also converge for any initial condition $\hV^0 \in \bR^n$. \citeauthor{tsitsiklis97analysis} proved convergence when $d = \dpi$ by showing that the projected Bellman operator is a contraction in $\dpi$-weighted norm. That is, for any $V, V' \in \bR^n$,
\begin{equation*}
\norm{\Proj_\dpi \cTpi V - \Proj_\dpi \cTpi V'}_\dpi \le \gamma \norm{V - V'}_\dpi,
\end{equation*}
from which an application of Banach's fixed point theorem allows us to conclude that $\hV^k \to \hVpi$. More formally, the result follows from noting that the induced operator norm of $\Proj_\dpi \Ppi  \leq 1$. The lack of a similar result when $d \ne \dpi$ explains the divergence, and is by now well-documented in the literature \citep{baird95residual}.

Independent of the convergence issues raised by off-policy learning, the fixed point of the Bellman equation with linear function approximation is also affected.  A metric for the quality of a value function $\hat{V}(s)$ is $\expect_{s \sim \dpi} [(\hat{V}(s) - \Vpi(s))^2] = \norm{\hat{V} - \Vpi}_\dpi$, the expected value prediction error sampling states from the stationary distribution of the policy evaluated. Under some conditions, we can bound the quality of the fixed point under off-policy data as a constant factor times the optimal prediction error $\norm{\Pi_{\dpi} \Vpi - \Vpi}_{\dpi}$.
\begin{restatable}{thm}{approximationError}[Based on \citenp{munos03error}]\label{thm:approximation_error}
Let $d \in \Delta(\cS)$ be some arbitrary distribution. Suppose that $\norm{\Pi_d P_{d_\pi}}_{d_\pi} < 1/\gamma$ and there is a fixed point $\hVpid$ to the projected Bellman equation $V := \Proj_d \cTpi V$. Then its approximation error in $\dpi$-weighted norm is at most
\begin{equation*}
\norm{\hVpid - \Vpi}_{\dpi} \le \frac{\norm{\Pi_{d} V^\pi - V^\pi}_{d_\pi}}{1- \gamma \norm{\Pi_d \Ppi}_{d_\pi}}.
\end{equation*}
Furthermore, this error is minimized when $d = \dpi$.
\end{restatable}

Theorem \ref{thm:approximation_error} is interesting because it suggests that $\dpi$ is also the optimal in the sense that it yields the smallest approximation bound. \citet{Kolter2011TheFP} showed that when Theorem \ref{thm:approximation_error} does not apply (because $\norm{\Pi_d \Ppi}_{d_\pi} \geq 1/\gamma$) it is possible to construct examples where the fixed point error is unbounded, even if $\norm{\Pi_d \Vpi - \Vpi}_\dpi = 0$ (i.e. cases where a perfect solution exists). Thus, no general bound on the quality of the off-policy fixed point exists.

Not only do we expect that improvements in off-policy learning should lead to more stable learning behavior, but also to the improved quality of the value functions which, in the control setting, should translate to increases in performance. In this paper we will study the covariant shift approach to off-policy learning, where updates in \eqnref{semigradient} are reweighting so as to induce a projection under $\dpi$.
 
\section{The Covariate Shift Approach}

Suppose that the stationary distributions $\dpi$ and $\dmu$ are known, and that states are updated according to a distribution $s \sim \dmu$. We use importance sampling (e.g. \citenp{precup00eligibility}) to define the update rule
\begin{equation*}
\theta \gets \theta + \alpha \frac{\dpi(s)}{\dmu(s)} \left [ r(s, a) + \gamma \phi(s')^\top \theta - \phi(s)^\top \theta \right ] \phi(s)^\top,
\end{equation*}
where as before $a \sim \mu(\cdot \cbar s), s' \sim P(\cdot \cbar s, a)$, is equivalent to applying the semi-gradient update rule \eqnref{semigradient} \emph{under the sampling distribution $\dpi$}. Further multiplying the update term by $\frac{\pi(a \cbar s)}{\mu(a \cbar s)}$, we recover (in expectation) the semi-gradient update rule for learning $\Vpi$, under the sampling distribution $\dpi$ \citep{hallak17consistent}. Thus, provided we reweighted updates correctly, we obtain a provably convergent off-policy algorithm.

The \emph{COP-TD} learning rule proposed by \citeauthor{hallak17consistent} learns the ratio $\ratio$ from samples. Although much of the original work is concerned with the combined learning dynamics of the value and the ratio, we will focus on the process by which this ratio is learned.

Similar to temporal difference learning, COP-TD estimates $\ratio$ by bootstrapping from a previous prediction. Given a step-size $\alpha > 0$, a ratio vector $c \in \bR^n$ and a sample transition $(s, a, s')$ where $s \sim \dmu$, $a \sim \mu(\cdot \cbar s)$, and $s' \sim P(\cdot \cbar s, a)$, COP-TD performs the following update:
\begin{equation}\label{eqn:cop_update_rule}
c(s') \gets c(s') + \alpha \left [ \frac{\pi(a \cbar s)}{\mu (a \cbar s)} c(s) - c(s') \right ] .
\end{equation}
Note that this update rule learns ``in reverse'' compared to TD learning. 
The expected behavior of the update rule is captured by the \emph{COP operator} $Y$:
\begin{equation*}
(Yc)(s') := \expect_{s \sim \dmu, a \sim \mu} \left [ \frac{\pi(a \cbar s)}{\mu (a \cbar s)} c(s) \cbar s' \right ] .
\end{equation*}
In vector notation, this operator is:
\begin{equation}\label{eqn:cop_operator}
Y c = \Ddmu^{-1} \Ppi^\top \Ddmu c .
\end{equation}
Any multiple of $\ratio$ is a fixed point of $Y$: $Y \beta \ratio = \beta \ratio$, for $\beta \in \bR$.
\citeauthor{hallak17consistent}, under the assumption that the transition matrix $\Ppi$ has a full set of real eigenvectors, give a partial proof that the iterates $c^{k+1} := Y c^k$ converge to such a fixed point. Our first result is to provide an alternative proof of convergence that does not require this assumption.
\begin{restatable}{thm}{copOperatorConvergence}\label{thm:cop_operator_convergence}
Suppose that $\Ppi$ defines an ergodic Markov chain on the state space $\cS$, and let $c^0 \in \Delta$. Then the process $c^{k+1} = Y c^{k}$ converges to $C \ratio$, where $C \in \bR$ is a positive scalar.
\end{restatable}
\begin{restatable}{cor}{normalizedOperatorConvergence}\label{cor:normalized_operator_convergence}
Suppose that the conditions of Theorem \ref{thm:cop_operator_convergence} are met. Define the \emph{normalized COP operator}
\begin{equation*}
(\bar Y c)(s') := \frac{\tilde c (s')}{\sum_{s} \tilde c(s)}  \qquad \tilde c := Y c .
\end{equation*}
Then the unique fixed point of the operator $\bar Y$ is the ratio $\ratio$, to which the process $c^{k+1} := \bar Y c^k$ converges.
\end{restatable}

\subsection{COP-TD with Linear Function Approximation}

The covariate shift method is called-for when the value function is approximated. Under these circumstances, one might expect that we also need to learn an approximate ratio $\hat c$. \citeauthor{hallak17consistent} consider the linear approximation
\begin{equation*}
\hat c (s) = \phi(s)^\top w,
\end{equation*}
where $w \in \bR^k$.\footnote{In practice, we may avoid negative $\hat c$'s by clipping them at 0.} This gives rise to a semi-gradient update rule similar to \eqnref{semigradient} but implementing \eqnref{cop_update_rule}:
\begin{equation*}
\tilde w \gets w + \alpha \left [ \frac{\pi(a \cbar s)}{\mu(a \cbar s)} \phi(s)^\top w - \phi(s')^\top w \right ] \phi(s')
\end{equation*}
and also followed by a projection step on the $\dmu$-weighted simplex $\Delta_{\Phi, \dmu}$ defined by the set $W_{\Phi, \dmu} := \{ u \in \bR^k : \sum_{s \in \cS} \dmu(s) \phi(s)^\top u = 1,  \phi(s)^\top u \geq 0 \}$:
\begin{equation*}
w \gets \argmin_{u \in W_{\Phi, \dmu}} \norm{u - \tilde w} .
\end{equation*}
The projection step ensures that the approximate ratio $\hat c$ corresponds to some distribution ration $\frac{d}{\dmu}$ for $d \in \cS$. The combined process is summarized by the normalized COP operator: ${\hat c}^{k+1} := \Proj_{\Delta_{\Phi, \dmu}} \Proj_{d} Y {\hat c}^k$, whose repeated application converges to some approximate ratio.

One interesting fact is that the semi-gradient update rule, which corresponds to a $d$-weighted projection, is by itself insufficient to guarantee the good behavior of the algorithm.
\begin{restatable}{lem}{fixedPointOfApproximateCop}
\label{lem:fixed_point_of_approximate_cop}
Let $Y$ be a symmetric COP-TD operator and $\Proj$ be the projection onto $\Phi$ in $L_2$ norm. If $\ratio$ is not in the span of $\Phi$, then $c = 0$ is the only solution to
\begin{equation*}
\Proj Y c = c .
\end{equation*}
\end{restatable}
Lemma \ref{lem:fixed_point_of_approximate_cop} argues that the normalization step is not only a convenience but is in fact necessary for the process to converge to anything meaningful. This is further validated by numerical experiments with general $\Ppi$, $\dmu$ and $d$ where we obseve that the repeated application of operator $\Proj_d Y$ either converges to 0 or diverges.

\section{A Practical COP-TD}

In this paper we are concerned with the application of COP-TD to practical scenarios, where approximating $\ratio$ is a must. As the following observations suggest, however, there are a number of limitations to COP-TD.

\noindent \textbf{Lack of contraction factor.} The operator $Y$ is not in general a contraction mapping. Hence, while the process $c^{k+1} \coloneqq Y c^k$ converges, it may do so at a slow rate, with greater variations in the sample-based case, and more importantly may be unstable when combined with function approximation.

\noindent \textbf{Hard-to-satisfy projection step.} In the approximate case, we saw that it is necessary to combine the COP operator to a projection onto the $\dmu$-weighted simplex. Although it is possible to approximate this projection step in an online, sample-based manner for linear function approximation  (\citeauthor{hallak17consistent} recommend constraining the weights to the simplex generated by a sufficiently large enough sample), no counterpart exists for more general classes of function approximations, making COP-TD hard to combine with neural networks.


In what follows we address these two issues in turn.

\subsection{The Discounted COP Learning Rule}

While repeated applications of the operator converge to $\ratio$, the operator is not in general a
contraction mapping, and its convergence profile is tied to the (usually unknown) mixing time of the Markov chain described by $\Ppi$. Our main contribution is the \emph{$\hgamma$-discounted COP-TD} learning rule, which recovers COP-TD for $\hgamma = 1$.

\begin{defn}
Let $c \in \bR^n$. For a step-size $\alpha > 0$, discount factor $\hgamma \in [0, 1]$, and sample $(s, a, s')$ drawn respectively from $\dmu, \mu$, and $P$, the $\hgamma$-discounted COP-TD learning rule is
\begin{equation}\label{eqn:discounted_cop}
c(s') \gets c(s') + \alpha \left [ \hgamma \frac{\pi(a \cbar s)}{\mu(a \cbar s)} c(s) + (1 - \hgamma) - c(s') \right ] .
\end{equation}
The corresponding operator is
\begin{equation*}
\Ygamma c := \hgamma Y c + (1 - \hgamma) e .
\end{equation*}
\end{defn}
By inspection, it is clear that $Y_1 = Y$. However, as we will see, the discounted COP-TD learning rule has several desirable properties compared to its undiscounted counterpart. We begin by characterizing the discounted COP operator.
\begin{defn}
For a given $\hgamma \in [0, 1]$, we define the \emph{discounted reset transition function} $\hPpi$ as:
\begin{equation*}
\hPpi := \hgamma \Ppi + (1 - \hgamma) e d_\mu^\top,
\end{equation*}
where $e d_\mu^\top$ is the matrix whose columns are all $\dmu$.
\end{defn}
The discounted reset transition function can be understood as a process which either transitions as
usual with probability $\hgamma$, or resets to the stationary distribution $\dmu$ with the remainder
probability. This is analogous to the perspective of the discount factor as a probability of terminating
\cite{white17unifying}, and is related to the constraint that arises in the dual formulation of the value function \cite{wang08dual}.

We denote by $\hdpi$ the stationary distribution satisfying $\hdpi = \hdpi \hPpi$. As an aside, the inclusion of the reset guarantees the ergodicity of the Markov chain defined by $\hPpi$.

\begin{restatable}{prop}{hdpStationaryDistribution}\label{prop:hdp_stationary_distribution}
The stationary distribution $\hdpi$ is given by
\begin{equation*}
\hdpi = (1 - \hgamma) (I - \hgamma \Ppi^\top)^{-1} \dmu,
\end{equation*}
where the sum
\begin{equation*}
(I - \hgamma \Ppi^\top)^{-1} := \sum_{t=0}^\infty (\hgamma \Ppi^\top)^t .
\end{equation*}
is convergent for $\hgamma < 1$.
\end{restatable}
Put another way, $\hdpi$ describes an exponentially weighted sum of $k$-step deviations from the behavior policy's stationary distribution $\dmu$, where the $k^{th}$ deviation corresponds to applying transition $\Ppi^\top$ $k$ times.

\begin{restatable}{lem}{YgammaFixedPoint}\label{lem:Ygamma_fixed_point}
For $\hgamma < 1$, the ratio $\hratio$ is the unique fixed point of the operator $\Ygamma$.
\end{restatable}

\begin{restatable}{thm}{discountedCopConvergence}\label{thm:discounted_cop_convergence}
Let $c^0 \in \Delta$. For $\hgamma < 1$ the process $c^{k+1} \coloneqq \Ygamma c^k$ converges to $\hratio$, where $\hdpi$ is the stationary distribution of the transition function $\hPpi$ corresponding to the given $\hgamma$.
\end{restatable}
One of the most appealing properties of the discounted operator (for $\hgamma < 1$) is that it neither requires normalization, or even positive initial values to guarantee convergence. As we shall see, this
greatly simplifies the learning process.

\subsection{Discounted COP with Linear Function Approximation}

The appeal of the COP-TD learning rule is that it can be applied online. The same remains true for
our discounted COP learning rule \eqnref{discounted_cop}. Naturally, when combined with function
approximation the same issue of norm arises: can our learning process itself be guaranteed to converge? The answer is yes, provided the discount factor is taken to be small enough.

To begin, let us assume sample transitions are drawn as $s \sim \dmu, a \sim \mu(\cdot \cbar s), s' \sim P(\cdot \cbar s, a)$, as before. Because $\dmu$ is the stationary distribution, $s' \sim \dmu$ also. The process we study is therefore described by the projected discounted COP operator $\Proj_{\dmu} \Ygamma$.

\begin{restatable}{lem}{concentrationCoefficient}\label{lem:concentration_coefficient}
The induced operator norm of the COP operator $Y^n$ is upper bounded by a constant $\sqrt{K_{\pi, \mu, n}}$, in the sense that
\begin{equation*}
\norm{Y^n}^2_{\dmu} \le K_{\pi, \mu, n} := \sup_{s' \in \cS} \sum_{s \in \cS} \frac{\dmu(s)}{\dmu(s')} \Ppi^n(s' \cbar s) .
\end{equation*}
Further, the series can be bounded by a constant,
\begin{equation*}
K_{\pi, \mu, n} \leq K_{\pi, \mu} := \norm{ \frac{\dmu(s)}{\dpi(s)} }_\infty \norm{ \frac{\dpi(s)}{\dmu(s)} }_\infty.
\end{equation*}
\end{restatable}
The term $K_{\pi, \mu, n}$ is a concentration coefficient similar to those studied by \citet{munos03error}. Intuitively, it measures the discrepancy in stationary distributions between two states that are ``close'' according to $\pi$, in the sense that $s'$ is reachable from $s$ in $n$ steps. When $\mu = \pi$, the sum simplifies to $\dpi(s')$ and this term is 1. 

We can make use of the concentration coefficient to provide a safe value of $\hgamma$ below which the discounted COP learning rule is convergent. Although most of our work concerns 1-step updates, we provide a slightly more general result on $n$-step methods here, based on known contraction results \citep{sutton18reinforcement} and the existing multi-step extension of COP-TD \citep{hallak17consistent}.

\begin{restatable}{thm}{nStepDiscountedContraction}\label{thm:n_step_discounted_contraction}
Consider the $n$-step discounted COP operator $\Ygamma^n$. Then for any $c \in \bR^n$,
\begin{equation*}
\norm{\Y^n c - \hratio}_{\dmu} \le \hgamma^n \sqrt{K_{\pi, \mu, n}} \norm{c - \hratio}_{\dmu}
\end{equation*}
and in particular for $\hgamma < (K_{\pi, \mu, n})^{-1/2n}$, $\Ygamma^n$ is a contraction mapping. Since $K_{\pi, \mu, n}$ is a bounded series, the exponential factor is guaranteed to dominate. As a result, there exists a value of $\hgamma < 1$ for which the projected $n$-step discounted COP operator $\Proj_{\dmu} \Y^n$ is a contraction mapping.
\end{restatable}

Theorem \ref{thm:n_step_discounted_contraction} shows that we can avoid the usual divergence issues with the learning rule \eqnref{discounted_cop} by taking a sufficiently small $\hgamma$. While these results are not altogether surprising (they mirror the case of value function approximation), we emphasize that there is no equivalent guarantee in the undiscounted case.

More generally, we are unlikely to be in the worst-case scenario achieving the concentration coefficient $K_{\pi, \mu, n}$ and, as our empirical evaluation will show,  divergence does not seem to be a problem even with large $\hgamma$. Yet, one may wonder whether it is relevant at all to learn an approximation to $\ratio$. Using Theorem \ref{thm:approximation_error} we argue that since the bound is continuous in the learning distribution $d$ we can expect improved performance even when the covariate shift is approximated for $\hgamma < 1$ or where a prediction error due to function approximation occurs.

Taken as a whole, our results suggest that incorporating the discount factor $\hgamma$ should improve the behavior of the COP-TD algorithm in practice.

\subsection{Soft Ratio Normalization}\label{sec:normalization}

Suppose we are given a function $c : \cS \to \bR$ differentiable w.r.t. its parameters for which we would like that
\begin{equation*}
\sum_{s \in \cS} \dmu(s) c(s) = 1 .
\end{equation*}
A common approach in deep reinforcement learning settings is to treat this as an additional loss to be minimized. In this section we also follow this approach, and consider minimizing the \emph{normalization loss}
\begin{equation}\label{eqn:loss_function_normalization}
\cL(c) := \tfrac{1}{2} \big (\sum_{s \in \cS} \dmu(s) c(s) - 1 \big )^2 .
\end{equation}
The gradient of this loss is
\begin{equation}\label{eqn:loss_gradient}
\grad \cL(c) = (\sum_{s \in \cS} \dmu(s) c(s) - 1) \sum_{s \in \cS} \dmu(s) \grad c(s) .
\end{equation}
We seek an unbiased estimate of this gradient. However, we cannot recover such an estimate with a single sample $s \sim \dmu$, in a classic case of the double-sampling problem \citep{baird95residual}. In particular, it is not hard to see that
\begin{equation*}
\expect_{s \sim \dmu} \left [ (c(s) - 1) \grad c(s) \right ] \ne \grad \cL(c) .
\end{equation*}
However, we can obtain such an estimate by considering $m \ge 2$ samples $s_1, \dots, s_m$ drawn from $\dmu$. The quantity
\begin{equation*}
\hat \grad(c) := \big(\tfrac{1}{m - 1} \sum_{i=2}^m c(s_i) - 1\big) \grad c(s_1)
\end{equation*}
is an unbiased estimate of the loss gradient \eqnref{loss_gradient}. In fact, as the following theorem states, we can do better by allowing each sample to play both roles in the estimate, and averaging the results.
\begin{restatable}{thm}{gradientEstimateNormalization}\label{thm:gradient_estimate_normalization}
Consider a differentiable function $c : \cS \to \bR$ and the loss function \eqnref{loss_function_normalization}. Given $s_1, \dots, s_m$ independent samples drawn from $\dmu$,
\begin{equation*}
\tfrac{1}{m} \sum_{i=1}^m \big ( \tfrac{1}{m - 1} \sum_{j \ne i} c(s_j) - 1\big ) \grad c(s_i)
\end{equation*}
is an unbiased estimate of $\grad \cL(c)$.
\end{restatable}
In our experimental section we will see that the normalization loss plays an important role in making COP-TD practical.

\begin{figure*}[tb!]
\begin{center}
\includegraphics[width=0.35\textwidth]{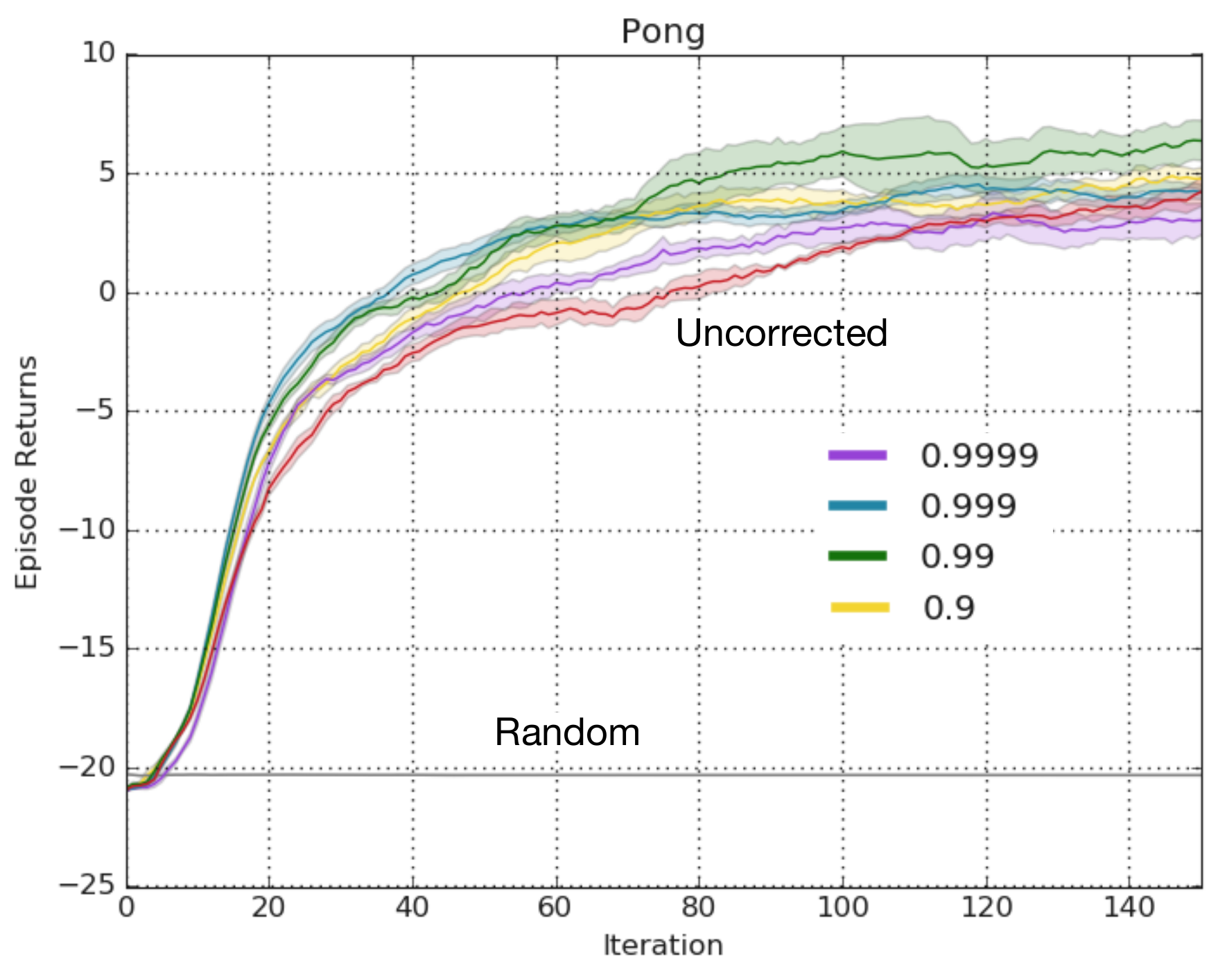}
\includegraphics[width=0.35\textwidth]{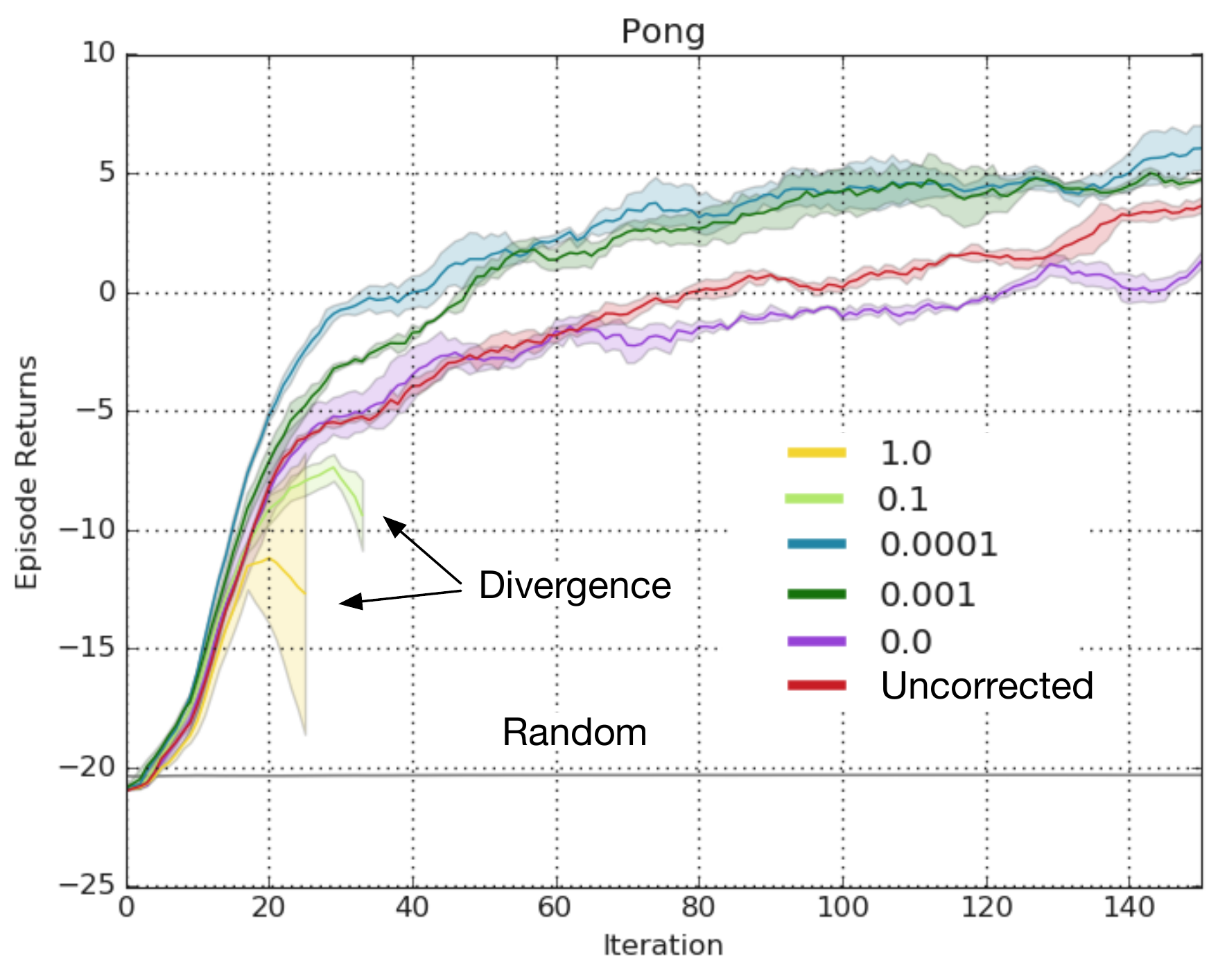}
\end{center}
\caption{ $\eta = 0.002$ with 5 seeds per run for 150 iterations. \textbf{Left.} Comparing discount factors in Pong. Using a discount factor gives a significant performance improvement. \textbf{Right.} Comparing normalization weights in Pong. Using normalization helps learning, but a large normalization weight causes divergence in the $c$ values.
\label{fig:sweep}}
\end{figure*}

\begin{figure*}[t]
\begin{center}
\includegraphics[width=.18\textwidth]{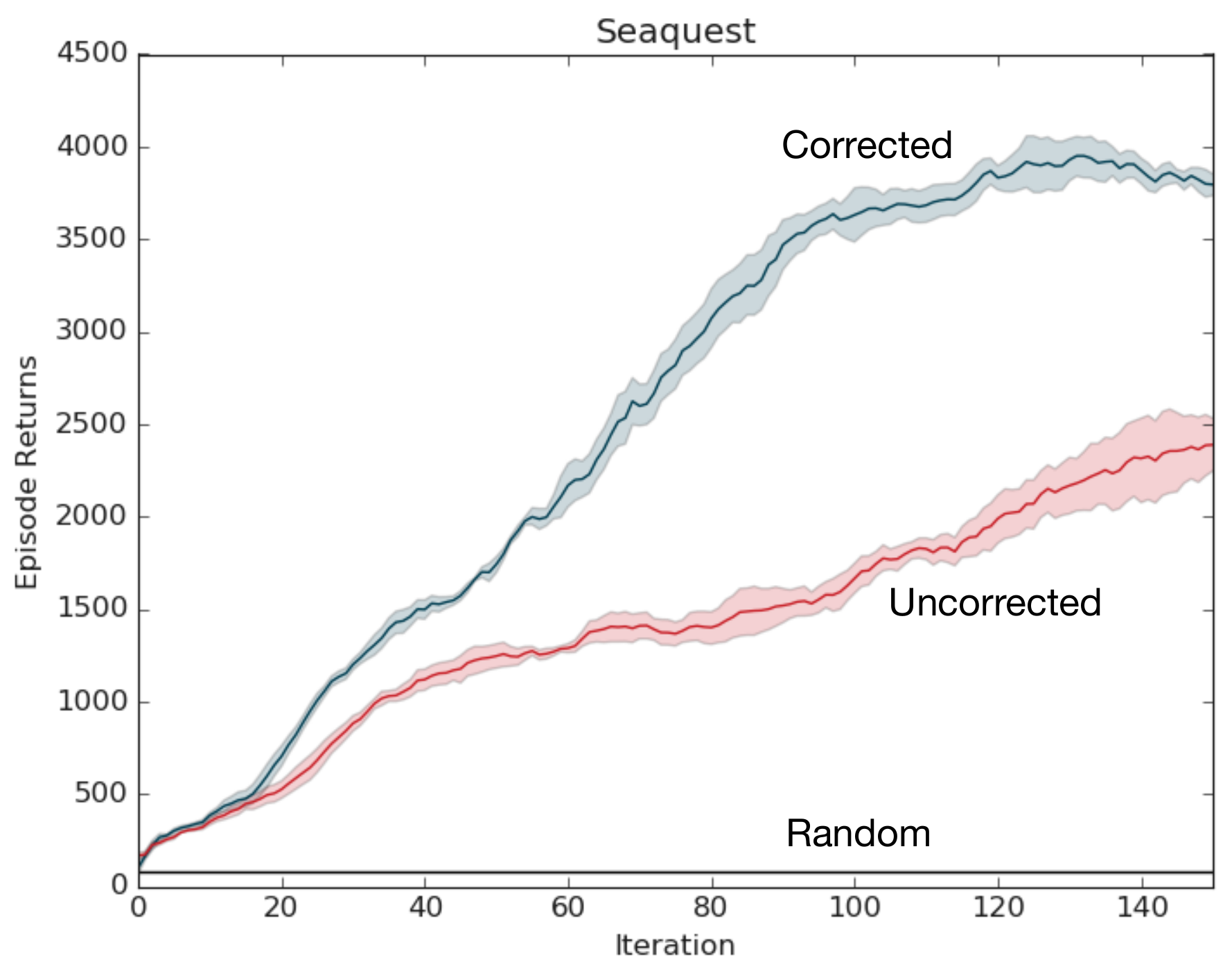}
\includegraphics[width=.18\textwidth]{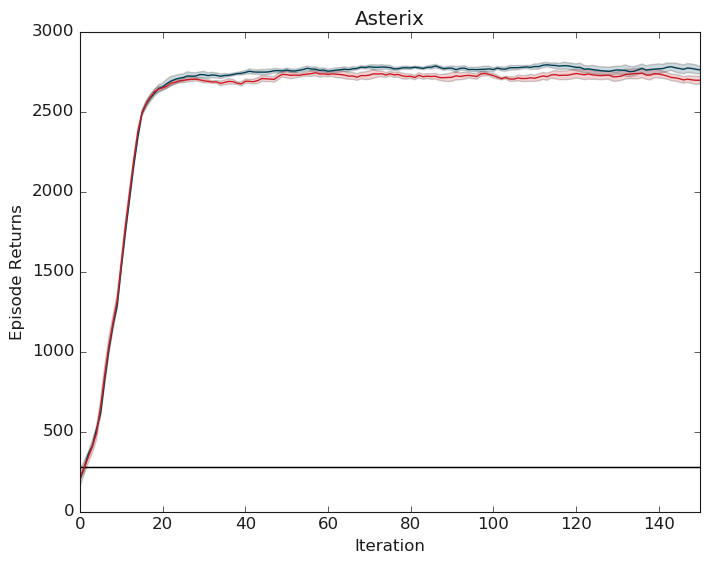}
\includegraphics[width=.18\textwidth]{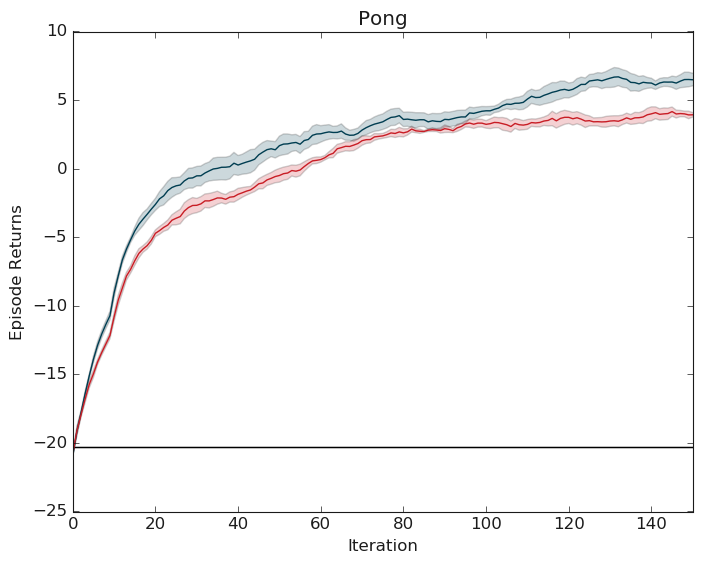}
\includegraphics[width=.18\textwidth]{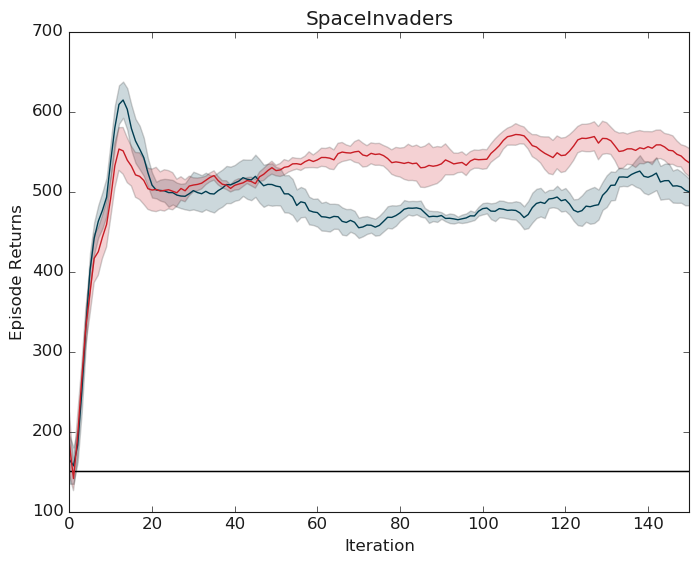}
\includegraphics[width=.18\textwidth]{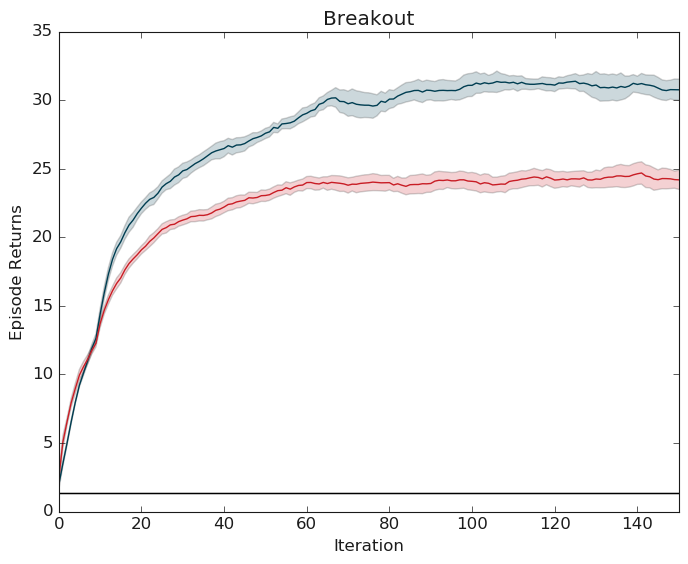}
\caption{ $\eta = 0.02$ with 3 seeds for 150 iterations. Performance of discounted COP-TD with a small target update period of 1000 and $\hgamma=0.99$ on 5 Atari 2600 games.
\label{fig:5_games}}
\end{center}
\end{figure*}

\section{Experimental Results}

In this section we provide empirical evidence demonstrating that our method yields useful benefits in an off-policy, deep reinforcement learning setting. In our experiments we use the Arcade Learning Environment (ALE) \citep{bellemare13arcade}, an RL interface to Atari 2600 games. We consider the single-GPU agent setup pioneered by \citet{mnih15human}. In this setup, the agent uses a replay memory (implemented as a windowed buffer) to store past experience, which it trains on continuously. As a result, much of the agent's learning carries an off-policy flavor.

We focus on a fixed behavior policy, specifically the uniformly random policy. We are interested in learning as good of a control policy as we can. That is, at each step the target policy is the greedy policy with respect to the predicted $Q$-values. While the theory we developed here applies to the policy evaluation case, we believe this setup to be a more practical and more stringent test of the idea. We emphasize that on the ALE, the uniformly random policy generates data that is significantly different from any learned policy; as a result, our experiments exhibit a high degree of off-policyness. To the best of our knowledge, we are the first to consider such a drastic setting.

\begin{figure*}[t]
\begin{center}
\includegraphics[width=.18\textwidth]{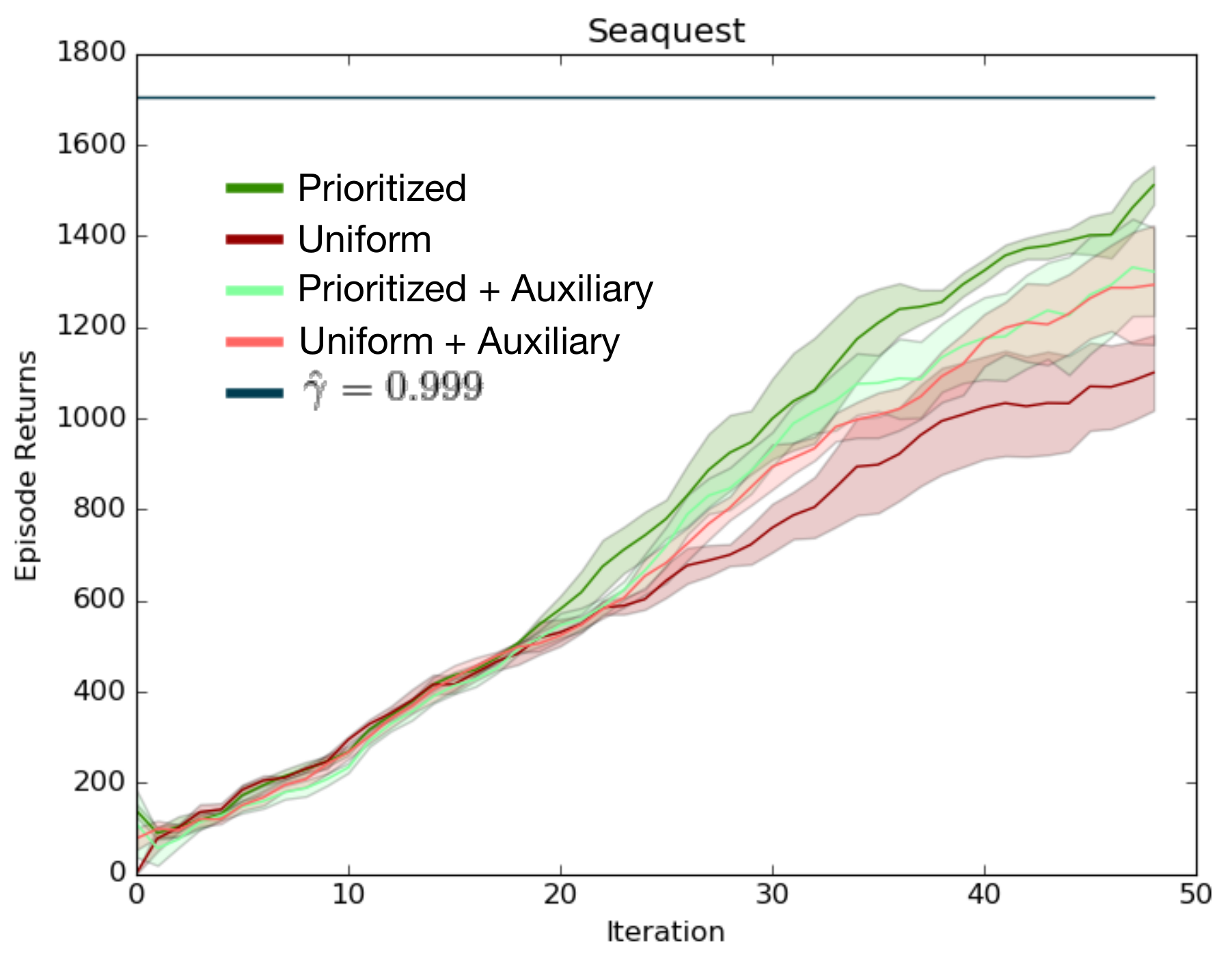}
\includegraphics[width=.18\textwidth]{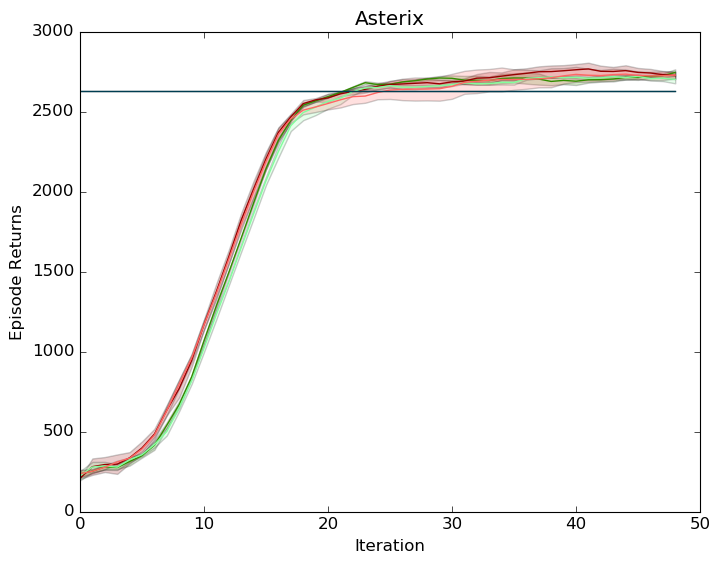}
\includegraphics[width=.18\textwidth]{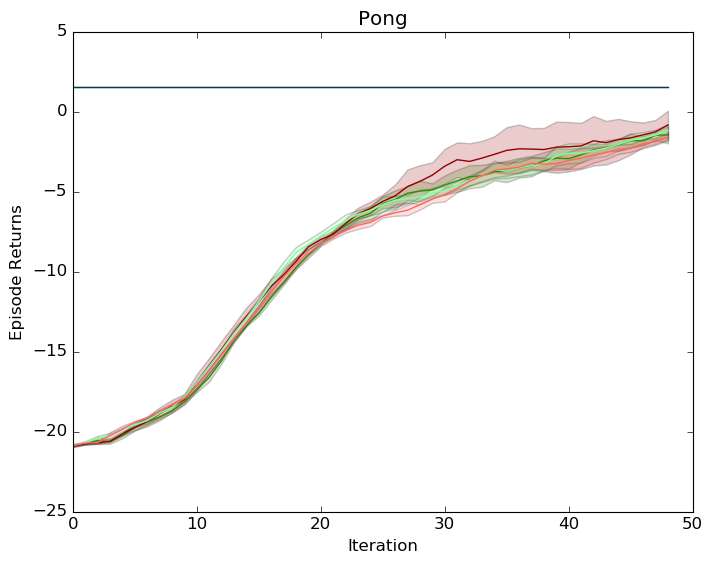}
\includegraphics[width=.18\textwidth]{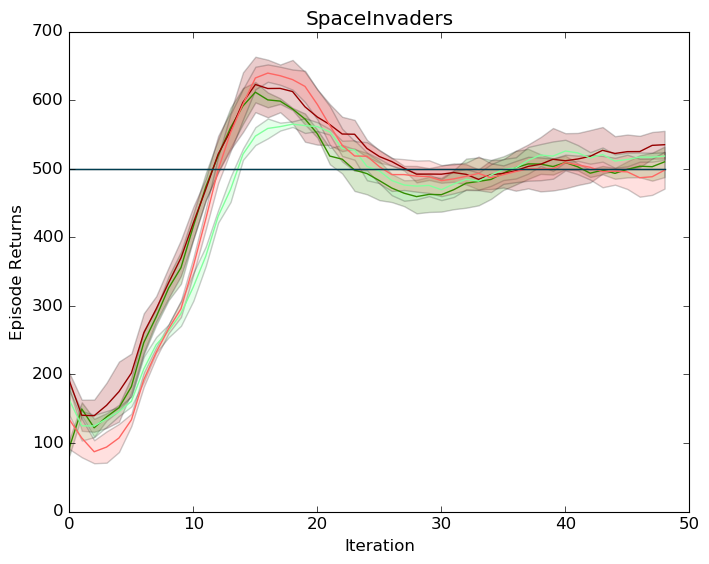}
\includegraphics[width=.18\textwidth]{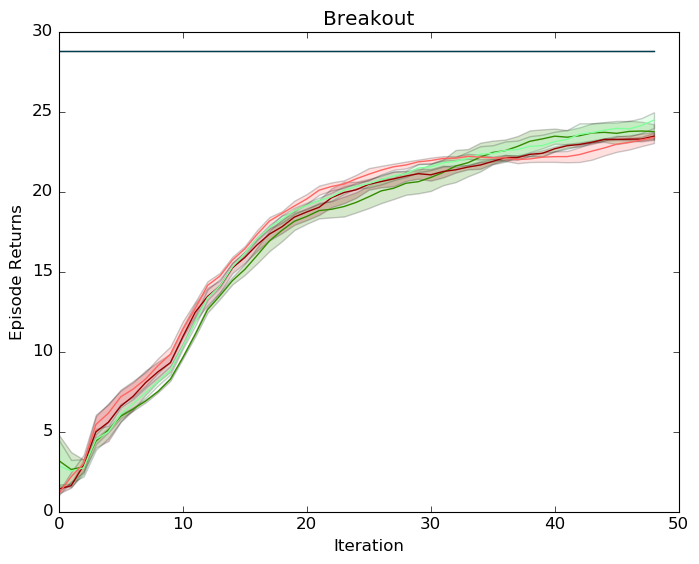}
\caption{ $\eta = 0.002$ with 3 seeds for 50 iterations. 4-way performance comparison using the discounted COP-TD loss as an auxiliary task and TD error prioritization as in \citep{schaul16prioritized}, blue line corresponds to the corrected agent with $\hgamma = 0.999$ at iteration 50.
\label{fig:5_games_aux}}
\end{center}
\end{figure*}

\begin{figure}[tb!]
\begin{center}
\includegraphics[width=0.35\textwidth]{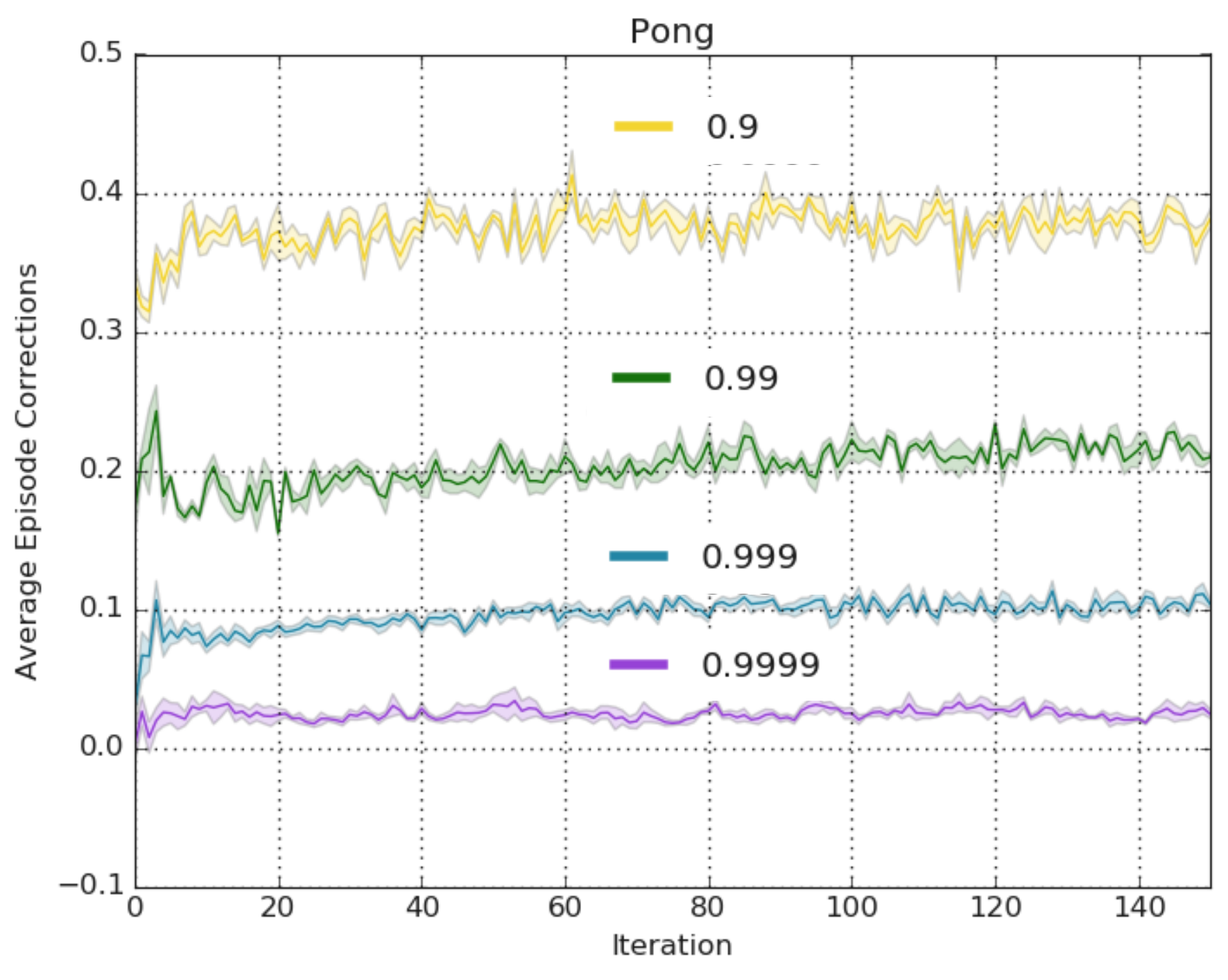}
\end{center}
\caption{ From same runs shown in Figure \ref{fig:sweep}, left. Average predicted ratio in evaluation episode for a set of $\hgamma$ in the game of Pong.
\label{fig:discounted_c}}
\end{figure}

\subsection{Implementation}

Our baseline is the C51 distributional reinforcement learning agent \cite{bellemare17distributional}, and we use published hyperparameters unless otherwise noted. We augment the C51 network by adding an extra head, the ratio model $c(s)$, to the final convolutional layer, whose role is to predict the ratio $\ratio$. The ratio model consists of a two-layer fully-connected network with a ReLU hidden layer of size 512. Whenever a correction term is used as a sampling priority or to compute a bootstrapping target, we clip negative outputs to 0. In what follows $\bar{\theta}$ denotes the parameters of the target network, which includes the ratio model.

In initial experiments we found that multiplicatively reweighting the loss function using covariate shifts hurt performance, likely due to larger gradient variance. Instead, to reweight sample transitions we use a prioritized replay memory \citep{schaul16prioritized} where priorities correspond to the approximate ratios of our model, which in expectation recovers the reweighting. These adjusted sampling priorities result in large portions of the dataset being mostly ignored (i.e. those unlikely under policy $\pi$); hence, the effective size of the data set is reduced and we risk overfitting. In our experiment we mitigated this effect by taking a larger replay memory size (10 million frames) than usual.

Identical to C51, the target policy $\pi_{\bar \theta}$ is the $\epsilon$-greedy policy with respect to the expected value of the distribution output of the target network. We set $\epsilon = 0.1$. The ratio model is trained by adding the squared loss
\begin{equation}\label{eqn:ratio_loss}
\eta \Big (\hgamma c_{\bar{\theta}}(s)\frac{\pi_{\bar{\theta}}(a|s)}{\mu(a|s)} + (1-\hgamma) - c_{\theta}(s')\Big)^2
\end{equation}
to the usual distributional loss of the agent, where $\eta > 0$ is a hyperparameter trading off the two losses. In experiments where we also normalize the ratio model, a third loss (with corresponding weight hyperparameter) is also added.
Preliminary experiments showed that learning the ratio with prioritized sampling led to stability issues, hence we train the ratio model by sampling transitions uniformly from the replay memory. Each training step samples two independent transition batches, prioritized and uniform for the value function and covariate shift respectively.

Since the training is done "backwards in time", no valid transition exists that would update the correction of an initial state $s_0$. This is similar to how there is no valid transition that updates the value of the terminal state in an episodic MDP. However, the distribution of any initial state $s_0$ is policy-independent, and so its ratio is $1$. As a result, we modify the loss \eqnref{ratio_loss} for initial states by replacing the bootstrap target with $1$. A more detailed analysis of our method in the episodic case is provided in the supplementary material.

\begin{figure*}[tb!]
\begin{center}
\includegraphics[width=6.5in]{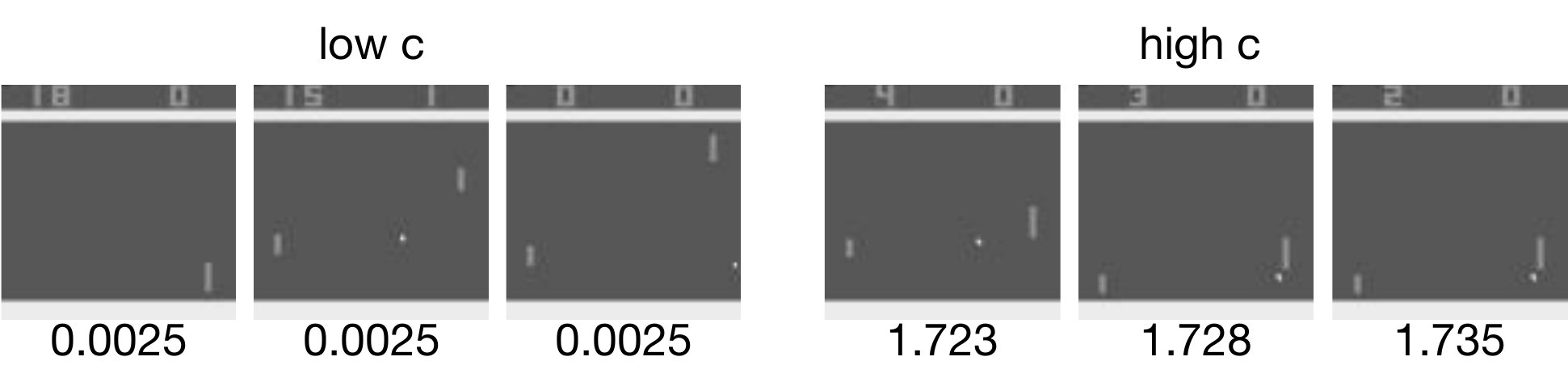}
\end{center}
\caption{Sample states (frames) encountered under the random policy, predicted either as relatively less likely under $\pi$ (low $c$) or relatively more likely under $\pi$ (high $c$). 
The experiment clipped the corrections at 0.0025 which was later found to be unnecessary. \label{fig:pong_high_low}}
\end{figure*}

\subsection{Discounting and Normalization}

We first study the effect of using the discounted COP update rule and/or normalization in the context of the game of Pong.
In Pong, the random agent achieves an average score close to -21, the minimum \citep{bellemare13arcade}. Figure \ref{fig:sweep}, left, compares the learning curves of various values of the discount factor $\hgamma$, the agent with no corrections and the random baseline using a ratio loss weight $\eta = 0.02$.
For discount factors not too large (all except $\hgamma = 0.9999$) better performance compared to the uncorrected baseline is achieved. Using normalization instead (Figure \ref{fig:sweep}, right) also improves performance. However, for high values of the normalization weight, we observed unstable and sometimes divergent behavior.
The runs which can be seen to stop mid plot had diverged in their $c$ outputs (Figure \ref{fig:undiscounted_c}, appendix).
We speculate that the reason for the divergence is higher variance in the loss function, and that a smaller step size might reduce such instabilities. Since using a discount factor proved more stable, has a slight performance advantage and is better understood theoretically than normalization, we will center the rest of the empirical evaluation around it.

In Figure \ref{fig:5_games} we report results for five Atari games chosen on the basis that a random agent playing these games explores the state space at least sufficiently to provide useful data for off-policy learning. We run C51 with discounted corrections with $\hgamma = 0.99$ and $\eta = 0.02$. We observe performance improvements in Seaquest, Breakout and Pong, no noticeable difference in Asterix and a small loss in Space Invaders.

\subsection{Auxiliary tasks and Prioritization}
One might wonder if the performance benefits observed are really due to sampling from a more on-policy distribution. Auxiliary tasks in the form of extra prediction losses have successfully been used to learn better representations and aid the learning of the value function \cite{Jaderberg2016ReinforcementLW}, \cite{Aytar2018PlayingHE}. To validate that the gains originate from correcting the off-policy data distribution as opposed to better representations, we show a modification of the previous experiment where the covariate shift was learned but not used. We also compare how our proposed prioritization scheme compares to the one originally proposed by \citep{schaul16prioritized} who used a function of the TD error to set the priorities. A four-way comparison of auxiliary tasks and TD error prioritization is shown in Figure \ref{fig:5_games_aux} where $\eta = 0.002$. We note that neither using the covariate shift prediction as an auxiliary task nor using TD error based prioritized sampling seemed to make any difference in all the games except SpaceInvaders. Interestingly, the covariate shift auxiliary task in SpaceInvaders helped when uniform sampling was used but hurt under prioritization. 

\subsection{The Effect of the Discount Factor}
To better understand the effect of $\hgamma$ in the learned ratios Figure \ref{fig:discounted_c} shows the average predicted ratio over evaluation episodes (where the $\epsilon$-greedy policy is used instead of the uniform random policy) in the game of Pong for the same set of runs shown in Figure \ref{fig:sweep}, left. Too large a discount ($\hgamma = 0.9999$) causes a decrease in performance, (see Figure \ref{fig:sweep}, left) and one might expect that divergence of the ratios would be the cause. Surprisingly, we observe that $\hgamma = 0.9999$ show no signs of divergence. We emphasize is that the average episode ratio decays monotonically with $\hgamma$, hinting that there is a tendency for the ratios to collapse to 0 which overcomes any potential divergence issues.

\subsection{Qualitative Evaluating the Learned Ratios}

As an additional experiment, we qualitatively assessed the ratio $c(s)$ learned by our deep network. We generated 100,000 sample states by executing the random behavior policy on Pong. From these, we selected the top and bottom 50 states according to the ratio ($c$ value) predicted by an agent trained under the regime of the previous section for 50 million frames. Recall that that $c > 1$ means the network believes the state is more likely under $\pi$ than $\mu$, while when $c < 1$ the converse is true.

Figure \ref{fig:pong_high_low} shows the outcome of this experiment for the top 3 states in terms of $c$-value, and 3 low-$c$ states; additional frames are provided in the supplemental. While our results remain qualitative, we see a clear trend in the selected images. States that are assigned low $c$ correspond to those in which the opponent is about to score a point (2nd and 3rd images). The network also assigns a low $c$ to a state in which the opponent has scored a high number of points (18 out of a possible total of 21) compared to the agent's (0 out of 21). This is indeed an unlikely state under $\pi$: if the trained agent ties the computer opponent, on average, then we expect its score to roughly match that of the opponent.

By contrast, states that are likely under $c$ are those for in which the agent successfully returns the ball. These are naturally unlikely situations under $\mu$, which plays randomly, but likely under the more successful policy $\pi$, which has learned to avoid the negative reward associated with failing to return the ball. 

From this qualitative evidence we conclude that our model learns to clearly distinguish likely and unlikely sample transitions. We believe these results are particularly significant given the relative scarcity of off-policy methods of this kind in deep reinforcement learning.

\section{Conclusion}

In this paper we revisited \citeauthor{hallak17consistent}'s COP-TD algorithm and extended it to be applicable to the deep reinforcement learning setting. While these results on the Atari 2600 suite of games remain preliminary, they demonstrate the practicality of learning the covariate shift in complex settings. We believe our results further open the door to increased sample efficiency in deep reinforcement learning.

We emphasize that the instabilities observed when learning the covariate shift under prioritized sampling point to the importance of the data distribution used to learn the ratios. Which distribution is optimal will be the focus of future work. The covariate shift method is a ``backward'' off-policy method, in the sense that it corrects a mismatch between distributions based on past transitions. It would be interesting to combine our method to ``forward'' off-policy methods such as Retrace \citep{munos16safe}, which have also yielded good results on the Atari 2600 \citep{gruslys18reactor}. Then, it would be interesting to understand whether overfitting does occur due to a smaller effective replay size, and how this can be addressed. Finally, an exciting avenue would be extending the method to the more general case where multiple policies have generated off-policy data, which would allow COP-TD to be applied in the standard control setting.

\section{Acknowledgements}
We would like to thank Dale Schuurmans, James Martens, Ivo Danihelka, Danilo J. Rezende for insightful discussion. We also thank Jacob Buckman, Saurabh Kumar, Robert Dadashi and Nicolas Le Roux for reviewing and improving the draft.

\bibliographystyle{aaai}
\bibliography{off-policy}

\clearpage

\section{Supplementary Material}

\approximationError*
\begin{proof}
\begin{align*}
\Pi_{d} \Vpi  - \Vpi &= \Pi_d \Vpi - \hVpid + \hVpid - \Vpi \\
&= \gamma \Pi_d \Ppi \Vpi - \gamma \Pi_d \Ppi \hVpid + \hVpid - \Vpi \\
&= (I - \gamma \Pi_d \Ppi)  (\hVpid - \Vpi ) \\
\end{align*}

Under the spectral norm assumption that  $\norm{\Pi_d P_{d_\pi}}_{d_\pi} < 1/\gamma$, the matrix $ (I - \gamma \Pi_d \Ppi) $ is invertible. Hence,

\begin{align*}
\hVpid - \Vpi  &= (I - \gamma \Pi_d \Ppi)^{-1}  (\Pi_{d} \Vpi  - \Vpi) \\
&= \sum_{t=0}^\infty (\gamma \Pi_d \Ppi)^t  (\Pi_{d} \Vpi  - \Vpi)
\end{align*}
and
\begin{align*}
\norm{\hVpid - \Vpi}_\dpi  &\leq \sum_{t=0}^\infty \norm{(\gamma \Pi_d \Ppi)^t}_\dpi \norm{ (\Pi_{d} \Vpi  - \Vpi) }_\dpi \\
&\leq \sum_{t=0}^\infty \norm{\gamma \Pi_d \Ppi}_\dpi^t \norm{ (\Pi_{d} \Vpi  - \Vpi) }_\dpi \\
& = \frac{\norm{\Pi_{d} V^\pi - V^\pi}_{d_\pi}}{1- \gamma \norm{\Pi_d \Ppi}_{d_\pi}}. \\
&\leq \frac{\norm{\Pi_{d} V^\pi - V^\pi}_{d_\pi}}{1- \gamma \norm{\Pi_d}_{d_\pi}}.
\end{align*}
Where the second step uses the  the Submultiplicative matrix norm property, proving the inequality. Further, to show that $d = \dpi$ minimizes the error upper bound, we just need to see that both, $ \norm{\Pi_{d} V^\pi - V^\pi}_{d_\pi} $ and $\norm{\Pi_d}_\dpi$ are minimized when $d=\dpi$, where $\norm{\Pi_\dpi}_\dpi = 1$.
\end{proof}

\copOperatorConvergence*
\begin{proof}
Let $k$ be fixed, write $C := \sum_s \dmu(s) c^0(s)$ and $b_0 := \frac{\Ddmu c^0}{C}$, such that the entries of $b_0$ sum to 1. We expand the definition of $c^k$:
\begin{align*}
c^k &= Y c_{k-1} = \dots = Y^k c^0 \\
&= (\Ddmu^{-1} \Ppi^\top \Ddmu)^k c^0 \\
&= \Ddmu^{-1} (\Ppi^\top )^k \Ddmu c^0 . \\
&= \Ddmu^{-1} (\Ppi^\top )^k b_0 C .
\end{align*}
From standard convergence results from Markov chain theory \cite{meyn12markov}, we know that $(\Ppi^\top)^k v \to \dpi$ for any vector $v$ with nonnegative entries and for which $\sum_{s} v(s) = 1$.
If $e$ is the vector of ones, we have
\begin{equation*}
c^k = \Ddmu^{-1} [ \dpi + \epsilon_k e ] C,
\end{equation*}
where $\epsilon_k$ is bounded and $\epsilon_k \to 0$ as $k \to \infty$. Hence
\begin{align*}
c^{k+1} &= C \ratio + C \epsilon_k \Ddmu^{-1} e \\
&\to C \ratio \qquad \text{ as } k \to \infty \text { also}. \qedhere
\end{align*}
\end{proof} 

\fixedPointOfApproximateCop*
\begin{proof}
First we show that if vectors co-linear with $\ratio$ are not in the span of $\Phi$, $\alpha \ratio$ can't be a fixed point of $\Proj Y$. By. contradiction,
\begin{align*}
\norm{\alpha \ratio} = \norm{\Proj Y \alpha \ratio} = \norm{\Proj \alpha \ratio} <  \norm{\alpha \ratio}
\end{align*}
We now show that for a $c \neq 0$ and $c \neq \alpha \ratio$, $c$ can't be a fixed point. Using the fact that the spectral norm of symmetric matrix is it's largest eigenvalue, and that $Y$ is a matrix similar to a transitions matrix and thus has the same set of eigenvalues, the largest of which is $1$, then $\norm{Yx} \leq 1$. In particular, $\norm{Y \ratio} = \norm{\ratio}$ and  $\norm{Y c} \le \norm{c}$. Again, by contradiction,
\begin{align*}
\norm{ c} = \norm{\Proj Y c} = \norm{\Proj} \norm{Yc} < \norm{c}
\end{align*}
\end{proof}

\normalizedOperatorConvergence*
\begin{proof}
Trivially, since the operator $Y$ is linear, the normalized COP operator $\bar{Y}^n = \alpha_n Y^n$. Using Theorem \ref{thm:cop_operator_convergence} we can state that  $\bar{Y}^\infty c^0 = \alpha C \ratio$ and the normalization term ensures that $\alpha = \frac{1}{C}$.
\end{proof}

\hdpStationaryDistribution*
\begin{proof}
By definition, $\hPpi^\top = (1 - \hgamma) \Ppi^\top + \hgamma \dmu e^\top$. First note that $\hdpi$ as defined is a distribution over $\cS$: $\hdpi \ge 0$ and since $e^\top \Ppi^\top = e^\top$,
\begin{align*}
e^\top \hdpi &= (1 - \hgamma) e^\top (I - \hgamma \Ppi^\top)^{-1} \dmu \\
&= (1 - \hgamma) \sum_{i=0}^\infty \gamma^i e^\top \dmu \\
&= 1 .
\end{align*}
We make use of the fact that $e^\top \hdpi = 1$ to write
\begin{align*}
\hat{P}_\pi^T \hat{d}_\pi  &=   \hgamma P_\pi^T   (1-\hgamma) (I - \hgamma P_\pi^T)^{-1}d_\mu + (1-\hgamma) {d_\mu} e^\top \hdpi \\
                                        &=   \hgamma P_\pi^T  (1-\hgamma) \sum_{i=0}^{\infty} (\hgamma P_\pi^T)^i d_\mu + (1-\hgamma) {d_\mu} \\
                                        &=   (1-\hgamma) (I - \hgamma P_\pi^T)^{-1}d_\mu \\
                                        &= \hat{d}_\pi . \qedhere
\end{align*}
\end{proof}

\YgammaFixedPoint*
\begin{proof}
We will prove that $\hratio$ is a fixed point of $\Ygamma$. Its uniqueness will be guaranteed by noting that the $n$-step operator $\Ygamma^n$ is a contraction mapping (Theorem \ref{thm:n_step_discounted_contraction} below) and invoking Banach's fixed point theorem. 

First note that for any $x \in \bR^n$, if $\tfrac{x}{\dmu}$ denotes the vector of ratios then $\Ddmu \tfrac{x}{\dmu} = x$, and in particular $\Ddmu^{-1} \dmu = e$.
We write
\begin{align*}
\Ygamma \frac{\hat{d}_\pi}{d_\mu}  &=  \hgamma \Ddmu^{-1} P_\pi^T D_{d_\mu} \frac{\hat{d}_\pi}{d_\mu}   + (1-\hgamma) e  \\
\Ygamma \frac{\hat{d}_\pi}{d_\mu}  &=  \hgamma \Ddmu^{-1} P_\pi^T \hdpi   + (1-\hgamma) \Ddmu^{-1} d_\mu  \\
                                            &= {D_{d_\mu}}^{-1} ( \hgamma P_\pi^T  {\hat{d}_\pi}   + (1-\hgamma) d_\mu) \\
                                            &= {D_{d_\mu}}^{-1} ( \hgamma P_\pi^T  {\hat{d}_\pi}   + (1-\hgamma) d_\mu e^\top \hdpi) \\
                                            &= {D_{d_\mu}}^{-1} (  \hat{P}_\pi^T \hdpi) \\
                                            &= \frac{\hat{d_\pi}}{d_\mu} . \qedhere \\
\end{align*}
\end{proof}

\discountedCopConvergence*
\begin{proof}
\begin{align*}
\Ygamma^k c^0 &= \sum_{i = 0}^{k-1} (\hgamma Y)^i (1-\hgamma) e + \hgamma^k Y^k c^0 \\ 
&=  (1-\hgamma)  \Ddmu^{-1} \ \sum_{i = 0}^{k-1} (\hgamma \Ppi^T)^i \dmu + \hgamma^k Y^k c^0
\end{align*}
as $ k \to \infty $,

\begin{align*}
\Ygamma^\infty c^0 &=  (1-\hgamma) \Ddmu^{-1}  \sum_{i = 0}^{\infty} (\hgamma Y)^i \dmu \\
&=  (1-\hgamma) \Ddmu^{-1} (I - \hgamma Y)^{-1} \dmu \\
&=  \hratio
\end{align*}\
\end{proof}

\concentrationCoefficient*
\begin{proof}
Let $ z(s') = \sum_{s} \frac{\dmu(s)}{\dmu(s')} \Ppi^n(s'|s) $, the normalization term needed to apply Jensen's inequality. Note that $K_{\pi, \mu, n} =  \sup_{s' \in \cS} z(s') $.

\begin{align*}
\norm{Y x}_\dmu^2 &=  \sum_{s'} d_\mu (s') \bigg(\sum_s \frac{d_\mu(s)}{d_\mu(s')} \Ppi^n(s'| s) x(s) \bigg)^2 \\
&=  \sum_{s'} d_\mu (s') z(s')^2 \bigg(\sum_s \frac{1}{z(s')} \frac{d_\mu(s)}{d_\mu(s')} \Ppi^n(s'| s) x(s) \bigg)^2 \\
&\leq  \sum_{s'} d_\mu (s') z(s') \sum_s \frac{d_\mu(s)}{d_\mu(s')} \Ppi^n(s'| s) x(s)^2  \\
&\leq K_{\pi, \mu, n}  \sum_{s}  \dmu(s)  x(s)^2  \sum_{s'} \Ppi^n(s'| s) \\
&= K_{\pi, \mu, n}  \sum_{s}  \dmu(s)  x(s)^2  = K_{\pi, \mu, n}  \norm{x}_\dmu^2
\end{align*} 
We now show that the series is upper bounded by a constant.
\begin{align*}
K_{\pi, \mu, n} &= \sup_{s' \in \cS} \sum_{s} \frac{\dmu(s)}{\dpi(s)} \frac{\dpi(s)}{\dmu(s')} \Ppi^n(s'|s) \\
 &\leq  \norm{ \frac{\dmu(s)}{\dpi(s)} }_\infty  \sup_{s' \in \cS} \sum_{s}  \frac{\dpi(s)}{\dmu(s')} \Ppi^n(s'|s) \\
 &=  \norm{ \frac{\dmu(s)}{\dpi(s)} }_\infty  \sup_{s' \in \cS}  \frac{\dpi(s')}{\dmu(s')} \\
&\leq  \norm{ \frac{\dmu(s)}{\dpi(s)} }_\infty \norm{ \frac{\dpi(s)}{\dmu(s)} }_\infty = K_{\pi, \mu}.
\end{align*} 
\end{proof}

\nStepDiscountedContraction*
\begin{proof}
\begin{align*}
\norm{\Y^n c - \hratio}_{\dmu}  &= \norm{\Y^n c - \Y^n \hratio}_{\dmu} \\
&= \norm{ \hgamma^n Y^n c - \hgamma^n Y^n \hratio}_{\dmu} \\
&\leq \hgamma^n \sqrt{K_{\pi, \mu, n}} \norm{c - \hratio}_{\dmu}
\end{align*} 
The rest follows easily.
\end{proof}

\gradientEstimateNormalization*
\begin{proof}
\begin{align*}
&\expect \bigg[\tfrac{1}{m} \sum_{i=1}^m \big ( \tfrac{1}{m - 1} \sum_{j \ne i} c(s_j) - 1\big ) \grad c(s_i) \bigg] = \\
&= \tfrac{1}{m} \sum_{i=1}^m  \expect \bigg[ \big ( \tfrac{1}{m - 1} \sum_{j \ne i} c(s_j) - 1\big ) \grad c(s_i) \bigg] \\
&= \tfrac{1}{m} \sum_{i=1}^m  \expect  \big [ c(s) - 1\big ] ] \expect[ \grad c(s) ] \\
&= \grad \expect [ ( c(s) - 1 \big )^2 ] ] =  \grad \cL(c) \\
\end{align*}
\end{proof}

\subsection{Episodic  COP-TD}
We show that for the episodic MDP case, there exists an equivalent $\Y$ operator that can be used to learn the Covariate Shift.

An episodic MDP has the same definition as a standard MDP except it has a starting distribution $d_0(s)$ denoting the probability of being in state $s$ at time step $0$ and that the transition function $P_\pi$ of any policy $\pi$ has a spectral radius smaller than $1$ , (e.i. the probability of reaching the terminal state after $\infty$ steps is $1$).

\begin{lem}
Let $d_0$ denote the starting state distribution vector and $d_\pi =  \sum_{i=0}^{\infty} (P_\pi^T)^i d_0 $ be the unnormalized probability vector of visiting a state. Then

\begin{equation*}
d_\pi = P_\pi^T d_\pi + d_0
\end{equation*}
\end{lem}

\begin{assumption}\label{single_starting_state}
The MDP has a single starting state $s_0$, $d_0(s_0) = 1$ which is impossible to return to $$P(s_0 | s, a) = 0 \forall s, a$$.
\end{assumption}
Although this assumption might seem strong at first, we note that it is easy for any MDP to satisfy it by including a beginning state that transitions to the original starting states according to the original starting distribution independent on the action taken. This assumption has two main desirable properties. First, $\ratio (s_0) = 1 \forall \pi, \mu$ and second $\Ddmu d_0 = d_0$.

\begin{lem}
Under Assumption \ref{single_starting_state} and letting $c(s_0) = 1$, the operator $Y$ defined as:
\begin{equation}\label{eqn:epi_cop_operator}
Y c = \Ddmu^{-1} \Ppi^\top \Ddmu c  + d_0.
\end{equation}
Can be written in expectation form as:
\begin{equation*}
(Yc)(s') := \expect_{s \sim \bar{\dmu}, a \sim \mu} \left [ \frac{\pi(a \cbar s)}{\mu (a \cbar s)} c(s) \cbar s' \right ] .
\end{equation*}
Where $\bar{d_\mu} = \frac{d_\mu}{|d_\mu|}$ denotes the (normalized) visitation probability vector.
\end{lem}

All the proofs in the episodic case follow closely that of their counterparts in the continuing case.

\begin{cor}
The discounted operator for the episodic case:
\begin{equation}\label{eqn:epi_discounted_cop_operator}
\Y c = \hgamma (\Ddmu^{-1} \Ppi^\top \Ddmu c  + d_0) + (1-\hgamma)e.
\end{equation}
Can also be written in expectation form as:
\begin{equation*}
(\Y c)(s') := \expect_{s \sim \bar{\dmu}, a \sim \mu} \left [ \hgamma \frac{\pi(a \cbar s)}{\mu (a \cbar s)} c(s) \cbar s' + 1-\hgamma \right ] .
\end{equation*}
\end{cor}

\begin{lem}
Under Assumption \ref{single_starting_state}, the operator $Y$ has a unique fixed point $\ratio$
\begin{equation*}
Y \ratio =  \ratio
\end{equation*}
\end{lem}

\begin{figure*}[tb!]
\begin{center}
\includegraphics[width=0.35\textwidth]{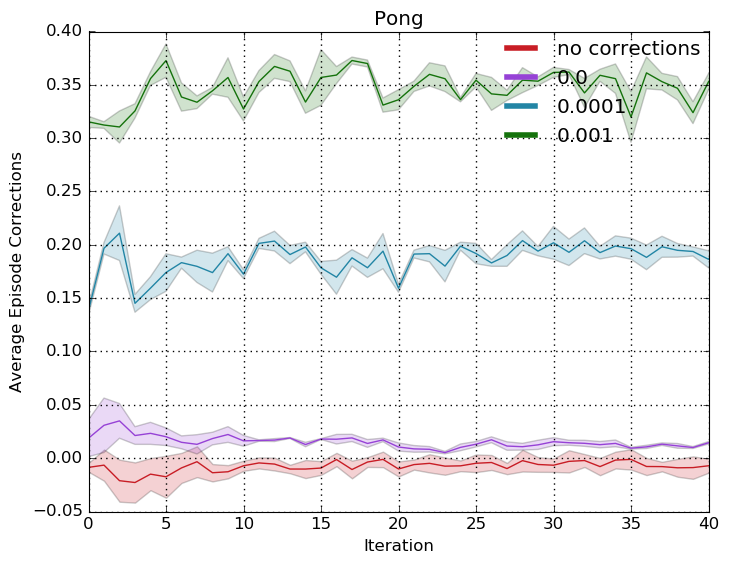}
\includegraphics[width=0.35\textwidth]{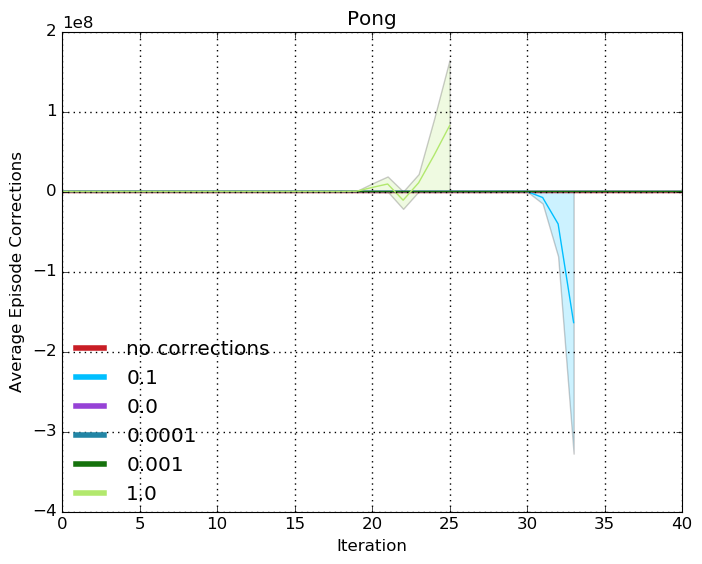}
\end{center}
\caption{\textbf{Left.} Compare the average c in the undiscounted case.
\label{fig:undiscounted_c}}
\end{figure*}

\begin{figure*}[t]
\begin{center}
\includegraphics[width=.18\textwidth]{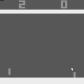}
\includegraphics[width=.18\textwidth]{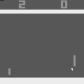}
\includegraphics[width=.18\textwidth]{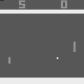}
\includegraphics[width=.18\textwidth]{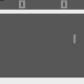}
\includegraphics[width=.18\textwidth]{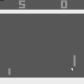}
\includegraphics[width=.18\textwidth]{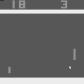}
\includegraphics[width=.18\textwidth]{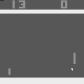}
\includegraphics[width=.18\textwidth]{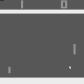}
\includegraphics[width=.18\textwidth]{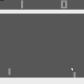}
\includegraphics[width=.18\textwidth]{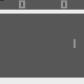}
\includegraphics[width=.18\textwidth]{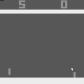}
\includegraphics[width=.18\textwidth]{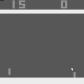}
\includegraphics[width=.18\textwidth]{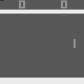}
\includegraphics[width=.18\textwidth]{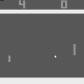}
\includegraphics[width=.18\textwidth]{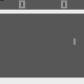}
\caption{Pong frames with high $c$-value.}
\end{center}
\end{figure*}

\begin{figure*}[t]
\begin{center}
\includegraphics[width=.18\textwidth]{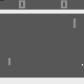}
\includegraphics[width=.18\textwidth]{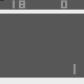}
\includegraphics[width=.18\textwidth]{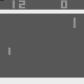}
\includegraphics[width=.18\textwidth]{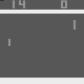}
\includegraphics[width=.18\textwidth]{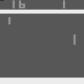}
\includegraphics[width=.18\textwidth]{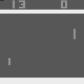}
\includegraphics[width=.18\textwidth]{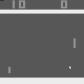}
\includegraphics[width=.18\textwidth]{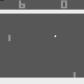}
\includegraphics[width=.18\textwidth]{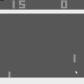}
\includegraphics[width=.18\textwidth]{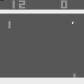}
\includegraphics[width=.18\textwidth]{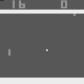}
\includegraphics[width=.18\textwidth]{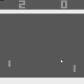}
\includegraphics[width=.18\textwidth]{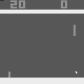}
\includegraphics[width=.18\textwidth]{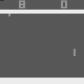}
\includegraphics[width=.18\textwidth]{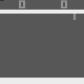}
\caption{Pong frames with low $c$-value.}
\end{center}
\end{figure*}

\end{document}